\documentclass[]{imsart}
\usepackage{amsthm,amsmath}
\usepackage[colorlinks,citecolor=blue,urlcolor=blue]{hyperref}

\usepackage{amsmath,amsfonts,amssymb,amsthm}

\usepackage{enumitem}

\RequirePackage{amsthm,amsmath,amsfonts,amssymb}
\RequirePackage[numbers]{natbib}
\RequirePackage[colorlinks,citecolor=blue,urlcolor=blue]{hyperref}
\RequirePackage{graphicx}
\usepackage{lineno,hyperref}

\startlocaldefs
\modulolinenumbers[5]

\usepackage[utf8]{inputenc}
\usepackage[english]{babel}

\usepackage[margin=1in]{geometry}
\usepackage{bm}
\usepackage{graphicx}
\usepackage{amssymb,amsmath,amsthm,mathrsfs}
\usepackage{epstopdf}
\usepackage{algorithm}
\usepackage{amsfonts,delarray}
\usepackage{algorithm,algcompatible,amsmath}
\usepackage{rotating,color}
\usepackage{subfigure}
\usepackage{multirow}
\usepackage{titlesec,textcase}
\usepackage{longtable}
\usepackage{bbm}
\usepackage{rotating}

\newtheorem{Theorem}{Theorem}[section]
\newtheorem{Lemma}[Theorem]{Lemma}

\theoremstyle{definition}

\RequirePackage[normalem]{ulem} 

\DeclareMathOperator*{\argmax}{arg\,max}

\definecolor{rp}{RGB}{83,54,106}

\def\boxit#1{\vbox{\hrule\hbox{\vrule\kern6pt\vbox{\kern6pt#1\kern6pt}\kern6pt\vrule}\hrule}}

\endlocaldefs

\begin{document}
\begin{frontmatter}
\title{Heterogeneous Dense Subhypergraph Detection}
\runtitle{Heterogeneous Dense Subhypergraph Detection}
\runauthor{M. Yuan and Z. Shang}
\begin{aug}
\author[A]{\fnms{Mingao} \snm{Yuan}\ead[label=e1]{mingao.yuan@ndsu.edu}}
\and
\author[B]{\fnms{Zuofeng} \snm{Shang}\ead[label=e2]{zshang@njit.edu}}
\address[A]{Department of Statistics,
North Dakota State University,
\printead{e1}}

\address[B]{Department of Mathematical Sciences,
New Jersey Institute of Technology,
\printead{e2}}
\end{aug}

\begin{abstract}
We study the problem of testing the existence of a heterogeneous dense subhypergraph. The null hypothesis corresponds to a heterogeneous Erd\"{o}s-R\'{e}nyi uniform random hypergraph and the alternative hypothesis corresponds to a heterogeneous uniform random hypergraph that contains a dense subhypergraph. We establish detection boundaries when the edge probabilities are known and construct an asymptotically powerful test for distinguishing the hypotheses. We also construct an adaptive test which does not involve edge probabilities, and hence, is more practically useful. 
\end{abstract}

\begin{keyword}[class=MSC2020]
\kwd[Primary ]{62G10}
\kwd[; Secondary ]{05C80,05C65}
\end{keyword}

\begin{keyword}
\kwd{detection boundary}
\kwd{heterogeneous uniform hypergraph}
\kwd{dense subhypergraph}
\kwd{asymptotically powerful test}
\end{keyword}

\end{frontmatter}

\section{Introduction}
\label{S:1}
Suppose $(\mathcal{V},\mathcal{E})$ is an \textit{undirected} $m$-uniform hypergraph on $N:=|\mathcal{V}|$ vertices with an edge set $\mathcal{E}$.
Each edge in $\mathcal{E}$ contains exactly $m$ distinct vertices. 
Without loss of generality, assume $\mathcal{V}=[N]:=\{1,2,\ldots,N\}$.
The adjacency tensor is an $m$-dimensional 0-1 symmetric array $A\in(\{0,1\}^N)^{\otimes m}$ such that $A_{i_1i_2\ldots i_m}=1$ if and only if $\{i_1,i_2,\ldots, i_m\}\in \mathcal{E}$. 
This implies that $A_{i_1i_2\ldots i_m}=0$ if $\{i_1,i_2,\ldots,i_m\}$ contains 
identical indexes, i.e., no self-loops are existent.
Symmetry is defined as $A_{i_1i_2\ldots i_m}=A_{j_1j_2\ldots j_m}$ if $i_1,i_2,\ldots,i_m$ is a permutation of $j_1,j_2,\ldots,j_m$. 
Let $p=\{p_{i_1i_2\ldots i_m}\in [0,1]: 1\leq i_1<i_2<\cdots i_m\leq N\}$ be a collection of edge-specific probability values. 
Let $\mathcal{H}_m(N,p)$ denote the
Erd\"{o}s-R\'{e}nyi $m$-uniform random hypergraph (see \cite{S02} for $m=2$).
Equivalently, $A_{i_1i_2\ldots i_m}$, $1\leq i_1<i_2<\cdots i_m\leq N$, are independent Bernoulli variables with 
\begin{equation}\label{bernoulli:H0}\mathbb{P}(A_{i_1i_2\ldots i_m}=1)=p_{i_1i_2\ldots i_m}.
\end{equation}
For a positive integer $n<N$, a subset $S\subset\mathcal{V}$ with $|S|=n$ and a scalar $\rho_S>1$, let $\mathcal{H}_m(N,p; n,\rho_S)$ denote an $m$-uniform random hypergraph with a dense subhypergraph $S$. Equivalently, $A_{i_1i_2\ldots i_m}$,
$1\leq i_1<i_2<\cdots<i_m\leq N$, are independent Bernoulli variables such that \begin{equation}\label{bernoulli:H1}
\mathbb{P}(A_{i_1i_2\ldots i_m}=1)=
\left\{
\begin{array}{cc}
p_{i_1i_2\ldots i_m}\rho_{S}, & i_1,i_2,\ldots,i_m\in S,\\
p_{i_1i_2\ldots i_m},& \textrm{otherwise}.
\end{array}\right.
\end{equation}
The assumption $\rho_{S}>1$ implies that the vertices within $S$ are more possibly connected, so $S$ can be viewed as an underlying dense subhypergraph. 
Since the Bernoulli probabilities in (\ref{bernoulli:H0}) and (\ref{bernoulli:H1}) are edge-specific, 
models $\mathcal{H}_m(N,p)$ and $\mathcal{H}_m(N,p; n,\rho_S)$ are both heterogeneous. Given $A$, we are interested in the following hypothesis testing problem:
\begin{eqnarray}\label{hypothesis}
\textrm{(null hypothesis) $H_0$:}&& \textrm{$A\sim \mathcal{H}_m(N,p)$} \nonumber\\
\textrm{(alternative hypothesis) $H_1$:}&& \textrm{there exists an $S\subset\mathcal{V}$ with $|S|=n$ and a $\rho_S>1$}\\ 
&&\textrm{such that
$A\sim \mathcal{H}_m(N,p; n,\rho_S)$}.\nonumber
\end{eqnarray}
The null hypothesis in (\ref{hypothesis}) says that $A$ follows an Erd\"{o}s-R\'{e}nyi $m$-uniform random hypergraph. The alternative hypothesis says that $A$ follows an $m$-uniform random hypergraph with an
underlying dense subhypergraph. 
 When $m=2$, \cite{BCHV19} derive detection boundaries for testing (\ref{hypothesis}). There is a lack of literature dealing with
 (\ref{hypothesis}) for arbitrary $m$ which will be investigated in this paper.

Given an observed hypergraph with an adjacency tensor $A$, a statistical test $T$ for testing (\ref{hypothesis}) is a function of $A$ such that $T=1$ if $H_0$ is rejected and $T=0$ otherwise. Define the risk of $T$ as
\[
\phi_N(T)=\mathbb{P}_0(T=1)+\max_{|S|=n}\mathbb{P}_S(T=0).
\]
Here $\mathbb{P}_0$ and $\mathbb{P}_S$ are the probability measures
under $H_0$ and $H_1$, respectively. If 
$\phi_N(T)\to0$ (or $\phi_N(T)\to1$),
we say that the test $T$ is asymptotically powerful (or asymptotically powerless).
In this paper, we provide 
conditions under which all tests for (\ref{hypothesis}) are asymptotically powerless. We also provide conditions under which
(\ref{hypothesis}) is asymptotically distinguishable. As an initial stage,
we consider the case of known $p$ and construct an asymptotically powerful test statistic. We then move forward to the more realistic unknown $p$ scenario and construct an adaptive test statistic. Our work is a hypergrahic extension of \cite{BCHV19}.
There are rich literature on classic homogeneous sub(hyper)graph detection or testing in which $p_{i_1i_2\ldots i_m}$ is constant.
For instance, \cite{LJY15, K11,ZHS06, BGK17,HWC17, CDKKR18,LZ20, YS21} 
proposed various detection algorithms. 
In stochastic block models,
various algorithms for detecting the underlying communities have been proposed in \cite{GD14,GD17,ACKZ15,B93,CK10,KBG17,LCW17,R09,RP13,KSX20,FP16,ALS16,ALS18}, and
\cite{YLFS20} established sharp boundaries for 
testing the existence of communities. More references could be found in the survey paper \cite{BTYZQ21}.
Nonetheless, problem (\ref{hypothesis}) is more challenging than the above homogeneous scenarios due to degree heterogeneity. 
 
The rest of the paper proceeds as follows. In Section \ref{sec2:subgraph}, we derive detection boundaries with all model parameters being known. In Section \ref{sec3:adaptive:test}, we construct an adaptive test to address the unknown edge probabilities. All additional proofs are deferred to Section \ref{sec4:proof}.

\section{Detection boundary when $p$ is known}\label{sec2:subgraph}
In this section, we derive detection boundaries
for testing (\ref{hypothesis}) when $p$ is known.
Assume that $n, N\rightarrow\infty$ and 
\begin{equation}\label{assum1}
\max\limits_{S\subset \mathcal{V},|S|=n}\max\limits_{\{i_1,\ldots,i_m\}\subset S}\rho_S^2p_{i_1i_2\ldots i_m}=o(1).
\end{equation}
Condition (\ref{assum1}) implies that the hypergraph is suitably sparse.
 Let $A_D=\sum_{i_1,\ldots,i_m\in D}A_{i_1i_2\ldots i_m}$
 which is the number of edges of the subhypergraph restricted to 
 the vertex set $D\subset \mathcal{V}$. Under $H_0$,
 the edge rate is $\mu_{D,m}=\mathbb{E}_0[A_D]/|D|^{(m)}$ with $|D|^{(m)}:=\binom{|D|}{m}$. Here $\mathbb{E}_0$ is the expectation under $H_0$. Denote $x\wedge y=\min\{x,y\}$, $x\vee y=\max\{x,y\}$ and $[x]_+=\max\{x,0\}$. Let 
 \[h(x)=(x+1)\log(x+1)-x,\,\,\,\,\,\, \forall x\in[0,1]\]
 which is related to the Kullback-Leibler divergence between two Bernoulli distributions.
Consider the following two scenarios regarding $n,N,m,\mu_{D,m}$.
\begin{itemize}[leftmargin=20mm]
\item[\bf Scenario 1.] There exist $\delta\in(0,0.5)$ and a positive sequence $\gamma_N=o(1)$ such that $n=O(N^{0.5-\delta})$ and
\begin{equation}\label{assum2}
\max\limits_{S\subset \mathcal{V},|S|=n,}\max\limits_{D\subset S,|D|<n/(N/n)^{\gamma_N}}\frac{|D|^{(m)}N}{N^{(m)}|D|}\frac{\mu_{D,m}}{\mu_{S,m}}\leq\delta,
\end{equation}
\begin{equation}\label{assum3}
\max\limits_{S\subset \mathcal{V},|S|=n,}\frac{1}{\mu_{S,m}}=o\left(\frac{n^{m-1}}{\log\frac{N}{n}}\right).
\end{equation}

\item[\bf Scenario 2.] Suppose $\log{n}=o(\log{N})$ and 
\begin{equation}\label{assum4}
\max\limits_{S\subset \mathcal{V},|S|=n}\frac{1}{\mu_{S,m}}=o\left(\frac{\log\frac{N}{n}}{\log n}\right)=o(1).
\end{equation}
\end{itemize}
The two scenarios above impose suitable assumptions on $n,N,m,\mu_{D,m}$.
For instance, (\ref{assum2}) says that the ratio of
$\mu_{D,m}$ to $\mu_{S,m}$ is well controlled when $D\subset S$ has small cardinality,
i.e., a small $D\subset S$ with large node degrees is nonexistent.
Hence, (\ref{assum2}) could be considered as a measure of heterogeneity.  Condition (\ref{assum3}) or (\ref{assum4}) requires the edge density in the dense subhypergraph being not too small, which implies that the underlying 
subhypergraph has enough signal to be detected. The conditions
degenerate to \cite{BCHV19} when $m=2$. For $m\geq3$ and $\mu_{S,m}=\mu_{S,2}$, condition (\ref{assum3}) is weaker than $m=2$.
The theorem below provides a circumstance
that the hypotheses in (\ref{hypothesis}) are asymptotically indistinguishable,
which is a generalization of \cite{BCHV19} to arbitrary fixed $m\geq3$. The indistinguishable regions for $m\geq3$ and $m=2$ are significantly different and the proof is technically much more involved. 
\begin{Theorem}\label{thm:1}
Under either Scenario I or Scenario II,
all tests are asymptotically powerless if there exists a constant $\epsilon\in(0,1)$ such that
\begin{equation}\label{cond:2}
\max\limits_{\substack{S\subset \mathcal{V}\\
|S|=n}}\max\limits_{D\subset S} \frac{|D|^{(m)}}{|D|}\frac{\mu_{D,m}h(\rho_S-1)}{\log\frac{N}{|D|}}\leq 1-\epsilon.
\end{equation}
\end{Theorem}
Condition (\ref{cond:2}) 
is equivalent to that for any $S\subset \mathcal{V}$ with $|S|=n$,
\begin{equation}\label{cond:2'}
h(\rho_S-1)\le \frac{1-\epsilon}{\max\limits_{D\subset S}\frac{|D|^{(m)}\mu_{D,m}}{|D|\log{\frac{N}{|D|}}}}.
\end{equation}
Theorem \ref{thm:1} says that if $\rho_S$ is close to one uniformly for $S$ of cardinality $n$, then the hypotheses in (\ref{hypothesis}) are  indistinguishable. When $m=2$, (\ref{cond:2}) degenerates to condition (5) in \cite{BCHV19}. To gain more insights about how (\ref{cond:2}) varies along with $m$, we suppose $\mu_{D,m}=\mu_{D,2}$ for $m\geq3$, i.e., the edge rates are constant along with $m$. Since $|D|^{(m)}$ increases in $m$, 
the RHS of (\ref{cond:2'}) decreases in $m$, i.e.,
the range of $\rho_S$ becomes smaller. This implies that 
testing (\ref{hypothesis}) for general $m\ge3$ becomes easier than $m=2$.
Proof of Theorem \ref{thm:1} proceeds by showing that the hypotheses in (\ref{hypothesis})
are asymptotically mutually contiguous under (\ref{cond:2}) which requires analyzing the likelihood ratio statistic.

\begin{proof}[Proof of Theorem \ref{thm:1}] 
For $i_1<i_2<\cdots<i_m$, denote $\mathbf{i}_m=\{i_1,\ldots,i_m\}\subset D$ and
denote $p_{i_1\ldots i_m}$ as $p_{\mathbf{i}_m}$. Let $\theta_{\mathbf{i}_m}(q)=\log\frac{q(1-p_{\mathbf{i}_m})}{p_{\mathbf{i}_m}(1-q)}$ for any $q\in(0,1)$, and $\Lambda_{\mathbf{i}_m}(\theta)=\log(1-p_{\mathbf{i}_m}+p_{\mathbf{i}_m}e^{\theta})$. 
Let $H_{p}(q)$ be the Kullback-Leibler divergence from $\text{Bern}(q)$ to $\text{Bern}(p)$ defined as
$H_{p}(q)=q\log\frac{q}{p}+(1-q)\log\frac{1-q}{1-p}$, for $p, q\in (0,1)$.
For a given subset $S\subset [N]$ with $|S|=n$, the likelihood ratio statistics based on (\ref{hypothesis}) is equal to 
\begin{eqnarray}\nonumber
L_S&=&\frac{\prod_{\substack{i_1<\cdots<i_m\\ \mathbf{i}_m\in S}}(\rho_Sp_{\mathbf{i}_m})^{A_{\mathbf{i}_m}}(1-\rho_Sp_{\mathbf{i}_m})^{1-A_{ \mathbf{i}_m}}}{\prod_{\substack{i_1<\cdots<i_m\\ \mathbf{i}_m\in S}}(p_{\mathbf{i}_m})^{A_{ \mathbf{i}_m}}(1-p_{\mathbf{i}_m})^{1-A_{ \mathbf{i}_m}}}=\exp\left\{\sum_{\mathbf{i}_m\in S}\left[A_{\mathbf{i}_m}\theta_{\mathbf{i}_m}(\rho_Sp_{\mathbf{i}_m})-\Lambda_{\mathbf{i}_m}(\theta_{\mathbf{i}_m}(\rho_Sp_{\mathbf{i}_m}))\right]\right\}.
\end{eqnarray}
Then it's easy to express the unconditional likelihood ratio statistic as
$L=\binom{N}{n}^{-1}\sum_{|S|=n}L_S$.
We adopt the likelihood ratio truncation skill used in \cite{AV14,BI13} to get
$\tilde{L}=\binom{N}{n}^{-1}\sum_{|S|=n}L_SI_{\Gamma_S}$,
where $I_E$ is an indicator function for event $E$ and 
\[\Gamma_S=\left\{\sum_{\mathbf{i}_m\in D}A_{\mathbf{i}_m}\theta_{\mathbf{i}_m}(\rho_Sp_{\mathbf{i}_m})\leq \sum_{\mathbf{i}_m\in D}\zeta_D\theta_{\mathbf{i}_m}(\rho_Sp_{\mathbf{i}_m}),\ all\ D\in E_S\right\},\]
with $\zeta_D$ provided in Lemma \ref{lem1}.
We will proceed by showing $\mathbb{E}_0\tilde{L}=1+o(1)$ and $\mathbb{E}_0\tilde{L}^2\leq 1+o(1)$, where $\mathbb{E}_0$ denotes expectation under $H_0$.

We begin with the first-order moment. It is easy to verify $\mathbb{E}_0\tilde{L}=\mathbb{P}_S\Gamma_S$. Note that $\theta_{\mathbf{i}_m}(\rho_Sp_{\mathbf{i}_m})\sim \log\rho_S$ uniformly for all $\mathbf{i}_m\in S$ with $|S|=n$ by a similar proof of equation (54) in \cite{BCHV19}. By Bennett's inequality, it follows that
\begin{eqnarray}\nonumber
1-\mathbb{P}_S\Gamma_S=\mathbb{P}_S\Gamma_S^c&\leq&\sum_{D\in E_S}\mathbb{P}\left(\sum_{\mathbf{i}_m\in D}A_{\mathbf{i}_m}>\zeta_D\sum_{\mathbf{i}_m\in D}p_{\mathbf{i}_m}\right)
\leq\sum_{D\in E_S}\exp\left(-\rho_S\sum_{\mathbf{i}_m\in D}p_{\mathbf{i}_m}h\left(\frac{\zeta_D}{\rho_S}-1\right)\right).
\end{eqnarray}
By Lemma \ref{lem2}, we get
\begin{eqnarray*}
1-\min_{S\subset \mathcal{V},|S|=n}\mathbb{P}_S\Gamma_S&\leq&\max_{S\subset \mathcal{V},|S|=n}\sum_{D\in E_S}\exp\left(-\rho_S\sum_{\mathbf{i}_m\in D}p_{\mathbf{i}_m}h\left(\frac{\zeta_D}{\rho_S}-1\right)\right)\\
&\leq&\max_{S\subset \mathcal{V},|S|=n}\sum_{k=m}^n\binom{N}{n}\exp\left(-\rho_S\sum_{\mathbf{i}_m\in D}p_{\mathbf{i}_m}h\left(\frac{\zeta_D}{\rho_S}-1\right)\right)\\
&\leq&\max_{S\subset \mathcal{V},|S|=n}\sum_{k=m}^n\exp\left(k\log\frac{ne}{k}-\rho_S\sum_{\mathbf{i}_m\in D}p_{\mathbf{i}_m}h\left(\frac{\zeta_D}{\rho_S}-1\right)\right)\\
&\leq&\sum_{k=m}^ne^{-k(c_n-1)}=o(1).
\end{eqnarray*}
Hence, $\mathbb{E}_0\tilde{L}=1+o(1)$.

We next consider the second-order moment. For $S_1,S_2\subset \mathcal{V}$ with $|S_1|=|S_2|=n$, let $D=S_1\cap S_2$. Clearly, we have the following
\begin{eqnarray*}
\mathbb{E}_0\tilde{L}^2&=&\binom{N}{n}^{-2}\sum_{|S_1|=n,|S_2|=n}\mathbb{E}_0L_{S_1}I_{\Gamma_{S_1}}L_{S_2}I_{\Gamma_{S_2}}\\
&=&\mathbb{E}_0\left(I_{\Gamma_{S_1}}I_{\Gamma_{S_2}}\exp\left(\sum_{i_m\in D}A_{\mathbf{i}_m}[ \theta_{\mathbf{i}_m}(\rho_{S_1}p_{\mathbf{i}_m})+\theta_{\mathbf{i}_m}(\rho_{S_2}p_{\mathbf{i}_m})]-\Lambda_{\mathbf{i}_m}(\theta_{\mathbf{i}_m}(\rho_{S_1}p_{\mathbf{i}_m}))-\Lambda_{\mathbf{i}_m}(\theta_{\mathbf{i}_m}(\rho_{S_2}p_{\mathbf{i}_m}))\right)\right)\\
&\leq&\mathbb{E}\left[I_{D\in E_{S_1}}\mathbb{E}_0\left(I_{\Gamma_{S_1}}\exp\left(\sum_{\mathbf{i}_m\in D}2A_{\mathbf{i}_m}[ \theta_{\mathbf{i}_m}(\rho_{S_1}p_{\mathbf{i}_m})-2\Lambda_{\mathbf{i}_m}(\theta_{\mathbf{i}_m}(\rho_{S_1}p_{\mathbf{i}_m}))\right)\right)\right]\\
& &+\mathbb{E}\left[I_{D\not\in E_{S_1}}\mathbb{E}_0\left(I_{\Gamma_{S_1}}\exp\left(\sum_{\mathbf{i}_m\in D}2A_{\mathbf{i}_m}[ \theta_{\mathbf{i}_m}(\rho_{S_1}p_{\mathbf{i}_m})-2\Lambda_{\mathbf{i}_m}(\theta_{\mathbf{i}_m}(\rho_{S_1}p_{\mathbf{i}_m}))\right)\right)\right]\\
&=&R_1+R_2.
\end{eqnarray*}
We shall show $R_1=o(1)$ and $R_2\leq1+o(1)$. Since $p_{\mathbf{i}_m}=o(1)$, we have
\[\Delta_{\mathbf{i}_m}^{(1)}:=\log\left(1+\frac{p_{\mathbf{i}_m}(\rho_{S_1}-1)^2}{1-p_{\mathbf{i}_m}}\right)\leq p_{\mathbf{i}_m}(\rho_{S_1}-1)^2.\]
Besides,
\[\mathbb{P}(|D|=k)\leq \exp\left(-k\left(\log\frac{Nk}{n^2}+O(1)\right)\right).\]
Consequently, by the definition of $E_{S_1}$ one has
\begin{eqnarray*}
R_2&\leq& \mathbb{E}\left[I_{D\not\in E_{S_1}}\exp\left(\sum_{\mathbf{i}_m\in D}\Delta_{\mathbf{i}_m}^{(1)}\right)\right]\\
&\leq&\mathbb{E}\left[I_{D\not\in E_{S_1}}\exp\left((1-\frac{\epsilon}{2})|D|\left(\log\frac{N|D|}{n^2}-\log\log\frac{N}{n}\right)\right)\right]\\
&\leq&1+\sum_{k=m}^n\exp\left(k\left(\log\frac{Nk}{n^2}-\log\log\frac{N}{n}\right)-k\left(\log\frac{Nk}{n^2}+O(1)\right)\right)=1+o(1).
\end{eqnarray*}

Note that $\frac{\log\zeta_D}{\log \rho_{S_1}}\sim \frac{\theta_{\mathbf{i}_m}(\zeta_Dp_{\mathbf{i}_m})}{\theta_{\mathbf{i}_m}(\rho_{S_1}p_{\mathbf{i}_m})}$ uniformly for $\mathbf{i}_m\in D$, $D\in E_{S_1}$ and $|S_1|=n$. For any $x\in [0,1]$ and $D\in E_{S_1}$, we have
\begin{eqnarray*}
R_1&\leq&\mathbb{E}\left[I_{D\in E_{S_1}}\mathbb{E}_0\left(\exp\left(\sum_{\mathbf{i}_m\in D}2\theta_{\mathbf{i}_m}(\rho_{S_1}p_{\mathbf{i}_m})\left(xA_{\mathbf{i}_m}+(1-x)\zeta_Dp_{\mathbf{i}_m} \right)-2\Lambda_{\mathbf{i}_m}(\theta_{\mathbf{i}_m}(\rho_{S_1}p_{\mathbf{i}_m}))\right)\right)\right],
\end{eqnarray*}
which is minimized at $x=\frac{\theta_{\mathbf{i}_m}(\zeta_Dp_{\mathbf{i}_m})}{2\theta_{\mathbf{i}_m}(\rho_{S_1}p_{\mathbf{i}_m})}\sim \frac{\log\zeta_D}{2\log \rho_{S_1}}$ by (50) in \cite{BCHV19}. Plugging this into the above equation yields
 \begin{eqnarray*}
R_1&\leq&\mathbb{E}\left[I_{D\in E_{S_1}}\mathbb{E}_0\exp\left(\sum_{\mathbf{i}_m\in D}\Delta_{\mathbf{i}_m}^{(2)}\right)\right]\leq \sum_{k=m}^n\exp\left(k(-2c_n+O(1))\right)=o(1),
\end{eqnarray*}
where $\Delta_{\mathbf{i}_m}^{(2)}=H_{p_{\mathbf{i}_m}}(\zeta_Dp_{\mathbf{i}_m})-2H_{\rho_{S_1}p_{\mathbf{i}_m}}(\zeta_Dp_{\mathbf{i}_m})$, the second inequality follows from the proof of Theorem 1 in \cite{BCHV19} and the last step follows from Lemma \ref{lem2}. Then the proof is complete.
\end{proof}

Next, we shall show that the condition (\ref{cond:2}) in Theorem \ref{thm:1} is also necessary for indistinguishability.
Define the hypergraphic scan statistic as
\begin{equation}\label{scan}
T_n=\max\limits_{D\subset \mathcal{V},|D|\leq n}T_D,\,\,\,\,\text{where}\,\,\,\, T_D=\frac{|D|^{(m)}}{|D|}\cdot\frac{\mu_{D,m}h\left([\frac{A_D}{\mathbb{E}_0[A_D]}-1]_+\right)}{\log\frac{N}{|D|}}.
\end{equation}
In the above, $\mathbb{E}_0[A_D]=\sum_{\textbf{i}_m\in D}p_{\textbf{i}_m}$ which is available given that $p$ is known.
For any $S\subset\mathcal{V}$ with $|S|=n$, let 
\[D^*_S=\argmax_{D\subset S} \frac{|D|^{(m)}}{|D|}\cdot\frac{\mu_{D,m}}{\log\frac{N}{|D|}}.\]

\begin{Theorem}\label{thm:2}
Suppose $n=o(N)$ and $\rho_S\mathbb{E}_0[A_{D^*_S}]\rightarrow\infty$. The scan test $T_n$ is asymptotically powerful if 
there exists a constant $\epsilon\in(0,1)$ such that
\begin{equation}\label{cond:3}
\min\limits_{\substack{S\subset \mathcal{V}\\
|S|=n}}\max\limits_{D\subset S} \frac{|D|^{(m)}}{|D|}\cdot\frac{\mu_{D,m}h(\rho_S-1)}{\log\frac{N}{|D|}}\geq 1+\epsilon.
\end{equation}
\end{Theorem}
Theorem \ref{thm:2} can be proved similarly as
Theorem 2 in \cite{BCHV19}. Specifically, it applies the Bennett's inequality to show that the $\gamma_N$-risk
of the scan test $T_n$ tends to zero. Theorems \ref{thm:1} and \ref{thm:2} together depict a detection boundary for
testing (\ref{hypothesis}).


\section{An adaptive test}\label{sec3:adaptive:test}
In practice, $p$ is often unknown so the
test statistic $T_n$ may not be applicable.
Instead, we will propose a new test which is adaptive to the value of $p$. 
This problem is challenging since $p$ contains $N^{(m)}$ unknown parameters
and estimation of these large amount of parameters seems infeasible. To reduce the amount of overparametrization, consider a special case that $p_{i_1\ldots i_m}=\prod_{k=1}^mW_{i_k}$ for an unknown vector $W=(W_1,\ldots,W_N)$ with $W_i\in[0,1]$.  When $m=2$, this is just the rank-1 model studied in \cite{BCHV19}. Moreover, assume that
\begin{equation}\label{assumrank1}
\left(\frac{W_{max}}{W_{min}}\right)^m=o\left(n^{\frac{m}{m+1}}\wedge\left[W_{min}^m\left(\frac{N}{n}\right)^{m-1-\delta_m}\right]\right),
\end{equation}
where $\delta_m=0$ for even $m$ and $\delta_m\in(0,1)$ for odd $m$, $W_{max}=\max\left\{W_1,\ldots,W_N\right\}$ and $W_{min}=\min\{W_1,\ldots,W_N\}$. Note that the RHS of (\ref{assumrank1}) 
converges to zero faster when $m$ is even.
Condition (\ref{assumrank1}) accommodates heterogeneity in the hypergraph. To see this, consider $n=\sqrt{N}$ and $W_i=\left(\frac{i}{N}\right)^{\frac{1}{k(m+1)}}$ in which $k\geq 4$ is a constant. Since $\sum_{i=1}^NW_i\asymp \frac{k(m+1)}{k(m+1)+1}N$,
it is easy to verify that the average degree of node $N$ (of order $N^{m-1}$) is approximately $N^{\frac{1}{k(m+1)}}$ times of the average degree of node 1 (of order $N^{m-1-\frac{1}{k(m+1)}}$).

When the edge probability $p$ is unknown, we need to modify the scan test. Essentially, we have to estimate $\mathbb{E}_0[A_D]$ for any subset $D\subset \mathcal{V}$ by $\widehat{p}_{D,m}$ defined as follows:
\begin{equation*}\label{modscan}
    \widehat{p}_{D,m}=\frac{1}{2^m}\left(A_{\mathcal{V}}^{\frac{1}{m}}-\left(A_{\mathcal{V}}-2A_{D,D^c}\right)^{\frac{1}{m}}\right)^m,
\end{equation*}
where
\[
A_{D,D^c}=\frac{1}{m!}\sum_{k=1}^{t_m}\binom{m}{2k-1}\sum_{\substack{\{i_1,\ldots,i_{2k-1}\}\subset D^c\\ \{i_{2k},\ldots,i_m\}\subset D}}A_{i_1\ldots i_m},\]
$t_m=\frac{m}{2}$ for even $m$, and $t_m=\frac{m+1}{2}$ for odd $m$, and $D^c=\mathcal{V}-D$. 
Define
\[\widehat{p}_{D,m}^*=\widehat{p}_{D,m}\vee \frac{|D|^m}{N^{m-1}}\log^{2m}\frac{N}{|D|}.\]
We then propose the following modified scan test statistic:
\begin{equation}\label{scan}
\widehat{T}_n=\max\limits_{D\subset \mathcal{V},|D|\leq n}\widehat{T}_D,\ \ \ \ \widehat{T}_D=\frac{\widehat{p}_{D,m}^*h\left([\frac{A_D}{\widehat{p}_{D,m}^*}-1]_+\right)}{|D|\log\frac{N}{|D|}}.
\end{equation}
Note that $\widehat{T}_n$ does not involve $p$ and hence is adaptive.
 Theorem \ref{thm:3} shows that $\widehat{T}_n$ is asymptotically powerful under the condition (\ref{cond:3}). 

\begin{Theorem}\label{thm:3}
Suppose $n=o(N)$ and $\rho_S\mathbb{E}_0[A_{D^*_S}]\rightarrow\infty$. 
If (\ref{cond:3}) holds,
then the modified scan test $\widehat{T}_n$ is asymptotically powerful.
\end{Theorem}
Based on Theorem \ref{thm:3}, the modified scan test $\widehat{T}_n$ still can achieve the detection boundary in Theorem \ref{thm:2}.
Moreover, the rate $n^{\frac{m}{m+1}}$ is optimal under Scenario I and Scenario 2 and it can't be improved.
 
The main ingredient in the proof of Theorem \ref{thm:3}
is to show that $\widehat{p}_{D,m}$ can accurately estimate $\mathbb{E}_0[A_D]$,
hence, $\widehat{T}_n$ will perform similarly as $T_n$.
Before proving Theorem \ref{thm:3}, let us provide a quick sketch
on the estimation accuracy of $\widehat{p}_{D,m}$.
Under assumption (\ref{assumrank1}), it can be shown that
for any $D\subset\mathcal{V}$, $\sum_{i\in \mathcal{V}-D}W_i-\sum_{i\in D}W_i$
is non-negative for large $N$.
Therefore, we have
\begin{eqnarray*}
2\sum_{i\in D}W_i&=&\sum_{i\in \mathcal{V}}W_i-\left(\sum_{i\in \mathcal{V}-D}W_i-\sum_{i\in D}W_i\right)
=\left[\left(\sum_{i\in \mathcal{V}}W_i\right)^m\right]^{\frac{1}{m}}-\left[\left(\sum_{i\in \mathcal{V}-D}W_i-\sum_{i\in D}W_i\right)^m\right]^{\frac{1}{m}}.
\end{eqnarray*}
Note that
\[
\left(\sum_{i\in \mathcal{V}}W_i\right)^m=m!\mathbb{E}_0A_{\mathcal{V}}+\sum_{\substack{|\{i_1,i_2,\ldots,i_m\}|\leq m-1\\
\{i_1,\ldots,i_m\}\subset \mathcal{V}}}W_{i_1}\cdots W_{i_m}
\]
and, by binomial formula,
\begin{eqnarray*}
&&\left(\sum_{i\in \mathcal{V}-D}W_i-\sum_{i\in D}W_i\right)^m\\
&=&\left(\sum_{i\in \mathcal{V}-D}W_i+\sum_{i\in D}W_i\right)^m-2\sum_{k=1}^{t_m}\binom{m}{2k-1}\left(\sum_{i\in \mathcal{V}-D}W_i\right)^{2k-1}\left(\sum_{i\in D}W_i\right)^{m-2k+1}\\
&=&m!\mathbb{E}_0A_{\mathcal{V}}+\sum_{\substack{|\{i_1,i_2,\ldots,i_m\}|\leq m-1\\
\{i_1,\ldots,i_m\}\subset \mathcal{V}}}W_{i_1}\cdots W_{i_m}-
2\sum_{k=1}^{t_m}\binom{m}{2k-1}\left(\sum_{i\in \mathcal{V}-D}W_i\right)^{2k-1}\left(\sum_{i\in D}W_i\right)^{m-2k+1}.
\end{eqnarray*}
By law of large number, it can be shown that 
$A_{\mathcal{V}}=(1+o_P(1))\mathbb{E}_0[A_{\mathcal{V}}]$,
and for any $D\subset\mathcal{V}$,
$A_{D,D^c}=(1+o_P(1))\mathbb{E}_0[A_{D,D^c}]$.
By assumption (\ref{assumrank1}),
it can be shown that, for any $D\subset\mathcal{V}$,
\[
\mathbb{E}_0[A_{D}]\gg\frac{1}{m!}\sum_{\substack{|\{i_1,i_2,\ldots,i_m\}|\leq m-1\\
\{i_1,\ldots,i_m\}\subset D}}W_{i_1}\cdots W_{i_m},
\]
which leads to that
\begin{eqnarray*}
\left(\sum_{i\in \mathcal{V}}W_i\right)^m&=&(1+o_P(1))m!A_{\mathcal{V}},\\
\left(\sum_{i\in \mathcal{V}-D}W_i-\sum_{i\in D}W_i\right)^m
&=&(1+o_P(1))m!A_{\mathcal{V}}-2(1+o_P(1))m!A_{D,D^c}.
\end{eqnarray*}
Hence, we get
\begin{eqnarray}\nonumber
\mathbb{E}_0[A_D]&=&\sum_{\substack{i_1<i_2<\cdots<i_m\\
\{i_1,\ldots,i_m\}\subset D}}W_{i_1}\cdots W_{i_m}\\ \nonumber
&=&\frac{1}{2^m m!}\left(2\sum_{i\in D}W_i\right)^m-
\frac{1}{m!}\sum_{\substack{|\{i_1,i_2,\ldots,i_m\}|\leq m-1\\
\{i_1,\ldots,i_m\}\subset D}}W_{i_1}\cdots W_{i_m}\\ \nonumber
&=&\frac{1+o_P(1)}{2^m}\left(A_{\mathcal{V}}^{\frac{1}{m}}-(A_{\mathcal{V}}-2A_{D,D^c})^{\frac{1}{m}}\right)^m-o(\mathbb{E}_0[A_{\mathcal{D}}])\\ \label{e0d}
&=&(1+o_P(1))\widehat{p}_{D,m}-o(\mathbb{E}_0[A_{\mathcal{D}}]).
\end{eqnarray}
Therefore, $(1+o_P(1))\widehat{p}_{D,m}=(1+o_P(1))\mathbb{E}_0[A_D]$,
i.e., $\widehat{p}_{D,m}$ is proven a good estimator of $\mathbb{E}_0[A_D]$.

\begin{proof}[Proof of Theorem \ref{thm:3}] 
Firstly, we control the type I error. Note that for $a\geq b$, $ah\left(\left[\frac{x}{a}-1\right]_+\right)\leq bh\left(\left[\frac{x}{b}-1\right]_+\right)$. We only need to prove
\begin{equation}\label{hatp1}
\max\limits_{n^{\frac{1}{m+1}}\leq|D|\leq n}\frac{\mathbb{E}_0[A_{D}]}{\widehat{p}_{D,m}^*}\leq 1+o_p(1).
\end{equation}
Define $\mathcal{D}=\left\{D\subset \mathcal{V}:n^{\frac{1}{m+1}}\leq|D|\leq n, \widehat{p}_{D,m}^*\leq \mathbb{E}_0[A_{D}] \right\}$. It suffices to prove (\ref{hatp1}) for $D\in\mathcal{D}$. By the definition of $\widehat{p}_{D,m}^*$, we have
\[\frac{|D|^m}{N^{m-1}}\log^{2m}\frac{N}{|D|}\leq \widehat{p}_{D,m}^*\leq \mathbb{E}_0[A_{D}]\leq \left(\sum_{i\in D}W_i\right)^m,\]
which implies $\sum_{i\in D}W_i\geq\frac{|D|}{N^{\frac{m-1}{m}}}\log^{2}\frac{N}{|D|}$. Besides, by assumption (\ref{assumrank1}), we have $W_{min}\geq \frac{1}{N^{\frac{m-1-\delta_m}{m}}}$. Hence,
\begin{eqnarray}\nonumber
\mathbb{E}_0[A_{D,D^c}]&=&\frac{1}{m!}\sum_{k=1}^{t_m}\binom{m}{2k-1}\left(\sum_{i\in \mathcal{V}-D}W_i\right)^{2k-1}\left(\sum_{i\in D}W_i\right)^{m-2k+1}\\  \nonumber
&\geq&\frac{1}{m!}\sum_{k=1}^{t_m}\binom{m}{2k-1}(N-|D|)^{2k-1}W_{min}^{2k-1}\left(\sum_{i\in D}W_i\right)^{m-2k+1}\\  \label{e0add}
&\geq&\frac{1}{m!}\sum_{k=1}^{t_m}\binom{m}{2k-1}\frac{(N-|D|)^{2k-1}}{N^{\frac{m-1-\delta_m}{m}(2k-1)}}\frac{\left(|D|\log^{2}\frac{N}{|D|}\right)^{m-2k+1}}{N^{\frac{m-1}{m}(m-2k+1)}}.
\end{eqnarray}
For even $m$ and a constant $c_1>0$, using the last term $k=\frac{m}{2}$ in (\ref{e0add}), we have
\[\mathbb{E}_0[A_{D,D^c}]\geq c_1|D|\log^{2}\frac{N}{|D|}.\]
For odd $m$ and a constant $c_2>0$, using the last term $k=\frac{m+1}{2}$ in (\ref{e0add}), we have
\[\mathbb{E}_0[A_{D,D^c}]\geq c_2N^{1+\delta_m}\geq c_2|D|\log^{2}\frac{N}{|D|}.\]
Take $c=\min\{c_1,c_2\}$. For a constant $c_0>0$, by Bennett's inequality we get
\begin{eqnarray*}\nonumber
&&\mathbb{P}\left(\min_{D\in\mathcal{D}}A_{D,D^c}-\mathbb{E}_0[A_{D,\bar{D}}]\leq -\sqrt{2c\left(\frac{1}{c}+c_0\right)\mathbb{E}_0[A_{D,D^c}]|D|\log\frac{N}{|D|}}\right)\\
&\leq&\sum_{k=m}^n\sum_{|D|=k}\exp\left(-\mathbb{E}_0[A_{D,D^c}]h\left(\sqrt{\frac{2c\left(\frac{1}{c}+c_0\right)|D|\log\frac{N}{|D|}}{\mathbb{E}_0[A_{D,D^c}]}}\right)\right)\\
&\leq&\sum_{k=m}^n\sum_{|D|=k}\exp\left(-c\left(\frac{1}{c}+c_0\right)|D|\log\frac{N}{|D|}\right)\\
&\leq&\sum_{k=m}^n\exp\left(k
\log\frac{Ne}{n}\right)\exp\left(-c\left(\frac{1}{c}+c_0\right)k\log\frac{N}{k}\right)=o(1).
\end{eqnarray*}
Here, we used the fact that $h(x)\sim \frac{x^2}{2}$ for $x=o(1)$. Consequently, we have
\[A_{D,D^c}=(1+o_p(1))\mathbb{E}_0[A_{D,D^c}],\]
uniformly for $D\in\mathcal{D}$. Obviously,
\[A_{\mathcal{V}}=(1+o_P(1))\mathbb{E}_0[A_{\mathcal{V}}].\]
By Lemma \ref{rank1lemma2}, we obtain
\[\frac{\widehat{p}_{D,m}^*}{\mathbb{E}_0[A_{\mathcal{V}}]}\geq \frac{\widehat{p}_{D,m}}{\mathbb{E}_0[A_{\mathcal{V}}]}=1+o_p(1).\]
Then by the proof of Theorem \ref{thm:2}, the type I error goes to zero.

Next, we control type II error.
Obviously, we have 
\[A_{D^*,D^{*c}}=(1+o_p(1))\mathbb{E}_1[A_{D^*,D^{*c}}],\ \ \ \ A_{\mathcal{V}}=(1+o_P(1))\mathbb{E}_1[A_{\mathcal{V}}].\]
By assumption (\ref{assumrank1}), since $\rho_SW_{min}^m\leq 1$, then $\rho_S\ll \left(\frac{W_{min}}{W_{max}}\right)^m\left(\frac{N}{n}\right)^{m-1-\delta_m}$. Then 
\begin{eqnarray*}
1\leq\frac{\mathbb{E}_1[A_{D^*,D^{*c}}]}{\mathbb{E}_0[A_{D^*,D^{*c}}]}\leq 1+\frac{\mathbb{E}_1[A_{D^*,C- D^*}]}{\mathbb{E}_0[A_{D^*,D^{*c}}]}=1+\rho_S\frac{\mathbb{E}_0[A_{D^*,C- D^*}]}{\mathbb{E}_0[A_{D^*,D^{*c}}]}.
\end{eqnarray*}
For even $m$, we have
\begin{eqnarray*}
\rho_S\frac{\mathbb{E}_0[A_{D^*,C- D^*}]}{\mathbb{E}_0[A_{D^*,D^{*c}}]}=O\left( \rho_S\frac{|D^*|(n-|D^*|)^{m-1}\vee |D^*|^{m}}{|D^*|(N-|D^*|)^{m-1}}\frac{W_{max}^m}{W_{min}^m}\right)=o(1).
\end{eqnarray*}
For odd $m$, we have
\begin{eqnarray*}
\rho_S\frac{\mathbb{E}_1[A_{D^*,C- D^*}]}{\mathbb{E}_0[A_{D^*,D^{*c}}]}=O\left( \rho_S\frac{(n-|D^*|)^{m}\vee |D^*|^{m}}{(N-|D^*|)^{m}}\frac{W_{max}^m}{W_{min}^m}\right)=o(1).
\end{eqnarray*}
Hence, $\mathbb{E}_1[A_{D^*,D^{*c}}]=(1+o(1))\mathbb{E}_0[A_{D^*,D^{*c}}]$. Similarly we can get $\mathbb{E}_1[A_{\mathcal{V}}]=(1+o(1))\mathbb{E}_0[A_{\mathcal{V}}]$.
Consequently, we have
\[A_{D^*,D^{*c}}=(1+o_p(1))\mathbb{E}_0[A_{D^*,D^{*c}}],\ \ \ \ A_{\mathcal{V}}=(1+o_P(1))\mathbb{E}_0[A_{\mathcal{V}}].\]
By Lemma \ref{rank1lemma2}, one has
$\widehat{p}_{D^*,m}=(1+o_p(1))\mathbb{E}_0[A_{D^*}]$. Hence,
\[\widehat{p}_{D^*,m}^*=(1+o_p(1))\mathbb{E}_0[A_{D^*}]\vee  \frac{|D^*|^m}{N^{m-1}}\log^{2m}\frac{N}{|D^*|}.\]

If $\widehat{p}_{D^*,m}^*=(1+o_p(1))\mathbb{E}_0[A_{D^*}]$, the proof is the same as Theorem \ref{thm:2}.

Next, we assume $\widehat{p}_{D^*,m}^*= \frac{|D^*|^m}{N^{m-1}}\log^{2m}\frac{N}{|D^*|}>\mathbb{E}_0[A_{D^*}]$. Note that $h^{-1}(x)\geq\sqrt{x}$. Then condition (\ref{cond:4}) implies that 
\[\rho_S>\sqrt{\frac{N^{m-1}}{|D^*|^{m-1}}\frac{1}{\log^{2m-1}\frac{N}{|D^*|}}}.\]
Recall that $W_{min}\geq \frac{1}{N^{\frac{m-1-\delta_m}{m}}}$. Hence,
\[\frac{\mathbb{E}_1[A_{D^*}]}{\widehat{p}_{D^*,m}^*}\geq \frac{\rho_S|D^*|^mW_{min}^m}{\frac{|D^*|^m}{N^{m-1}}\log^{2m}\frac{N}{|D^*|}}\geq\frac{\rho_S}{\log^{2m}\frac{N}{|D^*|}}\rightarrow\infty.\]

Since $A_{D^*}=(1+o_P(1))\mathbb{E}_1[A_{D^*}]$ and $h(x-1)\sim x\log x$ for $x\rightarrow\infty$, we have
\begin{eqnarray*}
\widehat{p}_{D^*,m}^*h\left(\left[\frac{A_{D^*}}{\widehat{p}_{D^*,m}^*}-1\right]_+\right)&\geq& \widehat{p}_{D^*,m}^*h\left(\left[\frac{\mathbb{E}_1[A_{D^*}]}{\widehat{p}_{D^*,m}^*}-1\right]_+\right)\\
&\geq&\mathbb{E}_1[A_{D^*}]\log\frac{\mathbb{E}_1[A_{D^*}]}{\widehat{p}_{D^*,m}^*}\\
&\geq&\mathbb{E}_1[A_{D^*}]\log\frac{\rho_S}{\log^{2m}\frac{N}{|D^*|}}\\
&\geq&\mathbb{E}_1[A_{D^*}]\log\rho_S=\mathbb{E}_0[A_{D^*}]h(\rho_S-1)\\
&\geq&(1+\epsilon)|D^*|\log\frac{N}{|D^*|}.
\end{eqnarray*}
Proof is complete.
\end{proof}



\section{Proof of additional lemmas}\label{sec4:proof}
In this section, we prove the lemmas. 
For a subset $S\subset \mathcal{V}$, define 
\[E_S=\left\{D\subset S: (\rho_S-1)^2\mathbb{E}_0[A_D]>\left(1-\frac{\epsilon}{2}\right)|D|\left(\log\frac{N|D|}{n^2}-\log\log\frac{N}{n}\right)\right\}.\]
The following preliminary lemmas can be found in \cite{BCHV19}.
\begin{Lemma}\label{lem1}
Under the conditions of Theorem \ref{thm:1}, for any $S\subset \mathcal{V}$ with $|S|=n$ and $D\in E_S$, there is an unique number $\zeta_D\geq1$ satisfying
\[
(1+\epsilon)\mathbb{E}_0[A_D]h(\zeta_D-1)=|D|\log\frac{N}{|D|},
\]
and $\theta_{i_1\ldots i_m}(\zeta_Dp_{i_1\ldots i_m})\leq2\theta_{i_1\ldots i_m}(\rho_Sp_{i_1\ldots i_m})$ for $i_1,\ldots,i_m\in D$.
\end{Lemma}

\begin{Lemma}\label{lem2}
Under the conditions of Theorem \ref{thm:1}, we have
\[c_n:=\min\limits_{S\subset \mathcal{V},|S|=n}\min\limits_{D\in E_S}\left((1-\epsilon)\frac{|D|^{(m)}}{|D|}\rho_S\mu_{D,m}h\left(\frac{\zeta_D}{\rho_S}-1\right)-\log\frac{n}{|D|}\right)\rightarrow\infty.\]
\end{Lemma}

\begin{Lemma}\label{lem3}
Under the conditions of Theorem \ref{thm:1}, for any $S\subset \mathcal{V}$ with $|S|=n$ and $D\in E_S$, we have
\[
\frac{\log\frac{n}{|D|}}{\log\frac{N}{n}}\left(\log\rho_S\vee 1\right)=o(1),\ \ \ \ \ \ \  \frac{\log\frac{N}{n}}{\log\rho_S}\rightarrow\infty.
\]
\end{Lemma}

\begin{proof}[Proof of Lemma \ref{lem3}] Firstly, we consider Scenario I. For any $S\subset \mathcal{V}$ with $|S|=n$, define $\eta_S\geq \rho_S$ such that
\[\frac{n^{(m)}}{n}\frac{\mu_{S,m}h(\eta_S-1)}{\log\frac{N}{n}}=1-\frac{2\epsilon}{3}.\]
Hence, by (\ref{assum3}), we get
\[h(\eta_S-1)< \frac{n\log\frac{N}{n}}{n^{(m)}\mu_{S,m}}=o(1),\]
which implies $\eta_S=1+o(1)$. Hence, $h(\eta_S-1)\sim\frac{(\eta_S-1)^2}{2}$. If $|D|<\frac{n}{\left(\frac{N}{n}\right)^{\gamma_N}}$, then we have
\begin{eqnarray*}
(\rho_S-1)^2\frac{\mathbb{E}_0[A_D]}{|D|}&\leq& (\eta_S-1)^2\frac{|D|^{(m)}}{|D|}\mu_{D,m}\\
&\leq&\delta (\eta_S-1)^2\frac{|S|^{(m)}}{|S|}\mu_{S,m}\\
&=&\left(1-\frac{2\epsilon}{3}\right)\delta (\eta_S-1)^2\frac{\log\frac{N}{n}}{h(\eta_S-1)}\\
&=&\left(1-\frac{2\epsilon}{3}\right)2\delta\log\frac{N}{n}\\
&\leq&\left(1-\frac{\epsilon}{2}\right)\left(\log\frac{N|D|}{n^2}-\log\log\frac{N}{n}\right).
\end{eqnarray*}
This is a contradiction to the fact that $D\in E_S$. As a result, we have $|D|\geq\frac{n}{\left(\frac{N}{n}\right)^{\gamma_N}}$. Then the desired result follows.

Under Scenario II, the proof is almost the same as in \cite{BCHV19}. We omit it here.
\end{proof}

\begin{Lemma}\label{rank1lemma1}
Under the assumption (\ref{assumrank1}), we have $|D^*|\geq n^{\frac{1}{m+1}}$.
\end{Lemma}

\begin{proof}[Proof of Lemma \ref{rank1lemma1}] 
Note that the function $f(x)=(x-1)\cdots(x-m+1)/(\log N-\log x)$ is increasing for $x\ll N$.
Suppose $|D|<n^{\frac{1}{m+1}}$ for any $D\subset\mathcal{V}$, then we have
\begin{eqnarray*}
\frac{|D|^{(m)}}{|D|}\frac{\mu_{D,m}}{\log\frac{N}{|D|}}&\leq&\frac{|D|^{(m)}}{|D|}\frac{W_{max}^m}{\log\frac{N}{|D|}}\leq \frac{|n^{\frac{1}{m+1}}|^{(m)}}{|n^{\frac{1}{m+1}}|}\frac{W_{max}^m}{\log\frac{N}{|n^{\frac{1}{m+1}}|}}\\
&<&\frac{|n^{\frac{1}{m+1}}|^{(m)}}{|n^{\frac{1}{m+1}}|}\frac{n^{\frac{m}{m+1}}W_{min}^m}{\log\frac{N}{|n^{\frac{1}{m+1}}|}}\leq \frac{|n^{\frac{1}{m+1}}|^{(m)}}{|n^{\frac{1+m}{m+1}}|}\frac{n^{\frac{m+m}{m+1}}W_{min}^m}{\log\frac{N}{|S|}}\\
&\leq& \frac{|S|^{(m)}}{|S|}\frac{W_{min}^m}{\log\frac{N}{|S|}}\leq \frac{|S|^{(m)}}{|S|}\frac{\mu_{S,m}}{\log\frac{N}{|S|}}.
\end{eqnarray*}
Hence, by the definition of $D^*$, we get $|D^*|\geq n^{\frac{1}{m+1}}$.
\end{proof}

\begin{Lemma}\label{function}
Suppose $n,k\geq2$ are fixed integers.
Let $f(x_1,\ldots,x_n)=
\frac{\sum_{i=1}^nx_i^k}{\left(\sum_{i=1}^nx_i\right)^k}$ be a function on $x_i\in [a,b], i=1,2,\ldots n$ and $0<a<b<1$. Then
\[\max_{x_i\in [a,b],i=1,\ldots,n}f(x_1,\ldots,x_n)\leq\frac{x_0(a^k-b^k)+nb^k}{\left[x_0(a-b)+nb\right]^k}\leq \frac{1}{(k-1)n^{k-1}}\sum_{t=1}^{k-1}\left(\frac{b}{a}\right)^t,\]
where
\[x_0=\frac{(a^k-b^k)nb-kn(a-b)b^k}{(k-1)(a-b)(a^k-b^k)}.\]
\end{Lemma}
\begin{proof}[Proof of Lemma \ref{function}] Consider $f(x_1,\ldots,x_n)$ as a function of $x_1$ with the rest arguments fixed. The derivative with respect to $x_1$ is equal to 
\[f_{x_1}=
\frac{k\sum_{i=2}^nx_i}{\left(\sum_{i=1}^nx_i\right)^{k+1}}\left(x_1^{k-1}-\left[\left(\frac{\sum_{i=2}^nx_i^k}{\sum_{i=2}^nx_i}\right)^{\frac{1}{k-1}}\right]^{k-1}\right).\]
When $x_1>\left(\frac{\sum_{i=2}^nx_i^k}{\sum_{i=2}^nx_i}\right)^{\frac{1}{k-1}}$, $f(x_1,\ldots,x_n)$ is increasing as a function of $x_1$. When $x_1<\left(\frac{\sum_{i=2}^nx_i^k}{\sum_{i=2}^nx_i}\right)^{\frac{1}{k-1}}$, $f(x_1,\ldots,x_n)$ is decreasing as a function of $x_1$. Hence, we get
\[
\max_{x_1\in [a,b]}f(x_1,\ldots,x_n)=
\frac{a^k+\sum_{i=2}^nx_i^k}{\left(a+\sum_{i=2}^nx_i\right)^k}\vee \frac{b^k+\sum_{i=2}^nx_i^k}{\left(b+\sum_{i=2}^nx_i\right)^k}.
\]
Repeating this procedure for each $x_2,\ldots,x_n$, we conclude
\[\max_{x_i\in [a,b],i=1,\ldots,n}f(x_1,\ldots,x_n)=\max_{0\leq x\leq n,x:integer}\frac{xa^k+(n-x)b^k}{\left[xa+(n-x)b\right]^k}\leq\max_{0\leq x\leq n}\frac{x(a^k-b^k)+nb^k}{\left[x(a-b)+nb\right]^k}.\]
Let $g(x)=\frac{x(a^k-b^k)+nb^k}{\left[x(a-b)+nb\right]^k}$ for $0\leq x\leq n$. The derivative of $g(x)$ is
\[g_x=\frac{(k-1)(a-b)(a^k-b^k)}{\left[x(a-b)+nb\right]^{k+1}}\left(x_0-x\right).\]
When $x<x_0$, $g(x)$ is increasing. When $x>x_0$, the $g(x)$ is decreasing. Hence, $\max_{x\in [0,n]}g(x)=g(x_0)$. Note that
\begin{eqnarray*}
x_0(a^k-b^k)+nb^k&=&\frac{(a^k-b^k)nb-kn(a-b)b^k}{(k-1)(a-b)}+nb^k\\
&=&\frac{n(a^{k-1}b+a^{k-2}b^2+\cdots+ab^{k-1}+b^{k})-nb^k}{k-1}\\
&=&\frac{na^k}{k-1}\sum_{t=1}^{k-1}\left(\frac{b}{a}\right)^t,
\end{eqnarray*}
\begin{eqnarray*}
x_0(a-b)+nbk&=&\frac{knb}{k-1}-\frac{knb^k}{(k-1)(a^{k-1}+a^{k-2}b+\cdots+b^{k-1})}\\
&=&\frac{kna}{k-1}\left(\frac{b}{a}-\frac{\left(\frac{b}{a}\right)^k}{\sum_{t=0}^{k-1}\left(\frac{b}{a}\right)^t}\right)\geq \frac{kna}{k-1}\left(1-\frac{1}{k}\right)=na.
\end{eqnarray*}
Hence, we have
\[g(x_0)\leq\frac{1}{n^ka^k}\frac{na^k}{k-1}\sum_{t=1}^{k-1}\left(\frac{b}{a}\right)^t=\frac{1}{(k-1)n^{k-1}}\sum_{t=1}^{k-1}\left(\frac{b}{a}\right)^t.\]
\end{proof}

\begin{Lemma}\label{rank1lemma2}
Under the assumption (\ref{assumrank1}), if 
\[A_{\mathcal{V}}=(1+o_P(1))\mathbb{E}_0[A_{\mathcal{V}}],\]
\[A_{D,D^c}=(1+o_P(1))\mathbb{E}_0[A_{D,D^c}],\]
uniformly for all $D\subset\mathcal{V}$ with $n^{\frac{1}{m+1}}\leq |D|\leq n$, then
\[\widehat{p}_{D,m}=(1+o_p(1))\mathbb{E}_0[A_{D}],\]
uniformly for all $D\subset\mathcal{V}$ with $n^{\frac{1}{m+1}}\leq |D|\leq n$. Under $H_1$, the above result still holds.
\end{Lemma}

\begin{proof}[Proof of Lemma \ref{rank1lemma1}] Define a function $f(x,y)=\left(x^{\frac{1}{m}}-(x-2y)^{\frac{1}{m}}\right)^m$ for $x\geq 2y$. At a fixed point $(a,b)$, the Taylor expansion is
\begin{equation}\label{taylor}
f(x,y)=f(a,b)+\frac{\partial f}{\partial x}(x^*,y^*)(x-a)+\frac{\partial f}{\partial y}(x^*,y^*)(y-b),
\end{equation}
where $x^*$ is between $a$ and $x$, $y^*$ is between $b$ and $y$ and
\[\frac{\partial f}{\partial x}(x,y)=\left(x^{\frac{1}{m}}-(x-2y)^{\frac{1}{m}}\right)^{m-1}\left(x^{\frac{1-m}{m}}-(x-2y)^{\frac{1-m}{m}}\right),\]
\[\frac{\partial f}{\partial x}(x,y)=\frac{2\left(x^{\frac{1}{m}}-(x-2y)^{\frac{1}{m}}\right)^{m}}{\left(x^{\frac{1}{m}}-(x-2y)^{\frac{1}{m}}\right)(x-2y)^{1-\frac{1}{m}}}.\]
Under the assumption (\ref{assumrank1}), 
$\left(\frac{W_{max}}{W_{min}}\right)^m\ll\left(\frac{N}{n}\right)^{m-1-\delta_m}W_{min}^m\ll\left(\frac{N}{n}\right)^{m-1-\delta_m}$. Then
\begin{equation}\label{e0av}
\frac{\mathbb{E}_0[A_{D,D^c}]}{\mathbb{E}_0[A_{\mathcal{V}}]}\leq\frac{\sum_{k=1}^{t_m}\binom{m}{2k-1}(N-n)^{2k-1}n^{m-2k+1}W_{max}^m}{N^{m}W_{min}^m}\ll \left(\frac{N}{n}\right)^{m-1-\delta_m}=o(1).
\end{equation}
Hence, we get $f(\mathbb{E}_0[A_{\mathcal{V}}],\mathbb{E}_0[A_{D,D^c}])=(1+o(1))\mathbb{E}_0[A_{\mathcal{V}}]$ uniformly in $D$. For $(x,y)=(A_{\mathcal{V}},A_{D,D^c})$ and $(a,b)=(\mathbb{E}_0[A_{\mathcal{V}}],\mathbb{E}_0[A_{D,D^c}])$, we also have
\[\frac{\partial f}{\partial x}(x^*,y^*)=1+o_p(1), \ \ \ \frac{\partial f}{\partial y}(x^*,y^*)=2+o_p(1).\]
By (\ref{taylor}), it's easy to check
\[\frac{2^m\widehat{p}_{D,m}}{\mathbb{E}_0[A_{\mathcal{V}}]}=\frac{f(A_{\mathcal{V}},A_{D,D^c})}{f(\mathbb{E}_0[A_{\mathcal{V}}],\mathbb{E}_0[A_{D,D^c}])(1+o(1))}=1+o_p(1).\]
It suffices to show $\mathbb{E}_0[A_{\mathcal{V}}]=2^m\mathbb{E}_0[A_D]$. To this end, by (\ref{e0d}) and (\ref{e0av}), we only need to prove
\begin{equation}\label{e0adgeq}
\mathbb{E}_0[A_{D}]\gg\frac{1}{m!}\sum_{\substack{|\{i_1,i_2,\ldots,i_m\}|\leq m-1\\
\{i_1,\ldots,i_m\}\subset D}}W_{i_1}\cdots W_{i_m},
\end{equation}
for any $|D|\geq n^{\frac{1}{m+1}}$. Note that
\[\mathbb{E}_0[A_{D}]=\frac{1}{m!}\left(\sum_{i\in D}W_i\right)^m-\frac{1}{m!}\sum_{\substack{|\{i_1,i_2,\ldots,i_m\}|\leq m-1\\
\{i_1,\ldots,i_m\}\subset D}}W_{i_1}\cdots W_{i_m}.\]
Obviously, $\sum_{\substack{|\{i_1,i_2,\ldots,i_m\}|\leq m-1\\
\{i_1,\ldots,i_m\}\subset D}}W_{i_1}\cdots W_{i_m}$ can be written as a summation of products of $\sum_{i\in D}W_i^k$, $k=1,2,\ldots,m$. By Lemma \ref{function} and assumption (\ref{assumrank1}), for the term $\sum_{i\in D}W_i^k$ with $k\geq 2$, we have
\[\frac{\sum_{i\in D}W_i^k}{(\sum_{i\in D}W_i)^k}=O\left(\frac{1}{|D|^{k-1}}\left(\frac{W_{max}}{W_{min}}\right)^{k-1}\right)=O\left(\frac{o(n^{\frac{k-1}{m+1}})}{n^{\frac{k-1}{m+1}}}\right)=o(1).\]
Hence, (\ref{e0adgeq}) holds.
Under $H_1$, the result can be similarly proved. 
\end{proof}

\end{document}